\documentclass[]{article}
\usepackage{proceed2e}

\usepackage{times}

\usepackage{amsmath}
\usepackage{amssymb}
\usepackage{url}
\usepackage{soul}
\usepackage{booktabs}
\usepackage[table]{xcolor}
\definecolor{lightblue}{RGB}{230,247,254}

\usepackage{graphicx}
\usepackage{pgfplots}
\usepackage{subfigure}

\newcommand{\citep}[1]{\cite{#1}}
\usepackage[round]{natbib}

\usepackage{xcolor,colortbl}

\definecolor{Gray}{gray}{0.85}
\definecolor{LightGray}{gray}{0.95}
\definecolor{LightCyan}{rgb}{0.88,1,1}
\newcolumntype{a}{>{\columncolor{Gray}}c}
\newcolumntype{b}{>{\columncolor{LightGray}}c}

\usepackage{mathrsfs}

\def\LL{\mathrm{LL}}
\def\MI{\mathrm{MI}}
\def\PL{\mathrm{Pen}}
\def\BIC{\mathrm{BIC}}

\usepackage{amssymb,amsmath,amsthm,amsfonts}
\usepackage{amsthm}
\newtheorem{Lemma}{Lemma}
\newtheorem{Theorem}{Theorem}
\newtheorem{Corollary}{Corollary}
\newcommand\idsia{\thanks{Istituto Dalle Molle di studi sull'Intelligenza Artificiale (IDSIA)} \  }
\newcommand\supsi{\thanks{Scuola universitaria professionale della Svizzera italiana (SUPSI)} \  }
\newcommand\usi{\thanks{Universit\`a della Svizzera italiana (USI) } \  }

\newcommand*\samethanks[1][\value{footnote}]{\footnotemark[#1] \  }

\newcommand\f[1]{$\mathbf{#1}$}

\title{Learning Bounded Treewidth Bayesian Networks with Thousands of Variables}

\author{
Mauro Scanagatta \\
IDSIA\idsia , SUPSI\supsi , USI\usi  \\
Lugano, Switzerland\\
\texttt{mauro@idsia.ch} \\
\And
Giorgio Corani\\
IDSIA\samethanks[1], SUPSI\samethanks[2], USI\samethanks[3] \\
Lugano, Switzerland \\
\texttt{giorgio@idsia.ch} 
\And
Cassio P. de Campos \\
Queen's University Belfast \\
Northern Ireland, UK \\
\texttt{c.decampos@qub.ac.uk} \\
\And
Marco Zaffalon\\
IDSIA\samethanks[1] \\
Lugano, Switzerland \\
\texttt{zaffalon@idsia.ch} 
}

\begin{document}

\maketitle

\begin{abstract}  

We present a method for learning treewidth-bounded Bayesian networks from data sets containing thousands of variables.
Bounding the treewidth of a Bayesian greatly reduces the complexity of inferences. 
Yet, being a global property of the graph, it considerably increases the difficulty of the learning process.
We propose a novel algorithm for this task, able to scale to large domains and large treewidths. 
Our novel approach consistently outperforms the state of the art on data sets with up to ten thousand variables. 
%Yet bounding the treewidth considerably increase the difficulty of the learning process.
%We propose a novel algorithm that is able to efficiently learn treewidth-bounded Bayesian networks and we 
%provide experiments with up to ten thousand variables. It outperforms the current state of the art approaches
%for learning with limited treewidth. 
%and it provably it is able to discover better networks than existing approximate methods for any value of treewidth up to 10. 
%Current state-of-the-art methods are able to learn network with a few hundred of variables. 
\end{abstract}

% Where we explain why we are wasting federal money and resources. 
\section{Introduction}

We consider the problem of structural learning of Bayesian networks with bounded treewidth, adopting a score-based approach.
Learning the structure of a bounded treewidth Bayesian network is a NP-hard problem \citep{korhonen2013}. 
It is therefore  unlikely the existence of an exact algorithm with  complexity polynomial in the number of variables $n$.
Yet learning Bayesian networks with bounded treewidth is deemed necessary to
allow exact tractable inference, since the worst-case inference complexity of
known algorithms is exponential in the treewidth $k$.

The topic has been thoroughly studied in the last years. 
A pioneering approach, polynomial in both the number of variables and the treewidth bound, has been proposed in \citep{elidan2008}.
It provides an upper-bound on the treewidth of the learned structure at each arc addition. The 
limit of this approach is that, as the number of variables increases, the bound becomes too large leading to sparse networks. 
% In the remainder of the paper we denote their approach as \emph{ED}. 
% 
% thus able to check  if the bound is violated. 

An \textit{exact} method has been proposed in \citep{korhonen2013}, which finds the highest-scoring network with the desired treewidth.
However, its complexity increases exponentially with the number of variables $n$.
Thus it has been applied in  experiments with up to only  15 variables. 

\cite{Parviainen2014} adopted an anytime integer linear programming (ILP).
If the algorithm is given enough time, it finds the highest-scoring network with bounded treewidth.
Otherwise it returns a sub-optimal DAG with bounded treewidth. 
The ILP problem has an exponential number of constraints in the number of variables, which limits its scalability.

\cite{NieMCJ14} proposed a more efficient anytime ILP approach with a polynomial number of constraints in the number of variables.
Yet they report that the quality of the solutions quickly degrades as the number of variables exceeds a few dozens and that no satisfactory solutions are found with data sets containing more than 50 variables.

Approximate approaches are therefore needed to scale to larger domains.
\cite{NieCJ15} proposed the approximated method S2.
It exploits the notion of k-tree, which is an undirected maximal graph with treewidth $k$. 
A Bayesian network whose moral graph is a subgraph of a k-tree has thus treewidth bounded by $k$.
S2 is an iterative algorithm. Each iteration consists of two steps: a) sampling uniformly a k-tree from the space of k-trees and b) recovering via sampling a high-scoring DAG whose moral graph is a sub-graph of the sampled k-tree.
The goodness of the k-tree is approximated by using a heuristic evaluation, called \emph{Informative Score}.
%However there is no guarantee that the DAG obtained starting given a randomly sampled k-tree will offer a good fit.
%
\cite{nieCJ16} further refines this idea, proposing an exploration guided via A* for finding the optimal k-tree with respect to the Informative Score. This algorithm is called S2+.
%. The goodness of the k-tree is approximated by using a heuristic evaluation, called \emph{Informative Score}. When a promising k-tree has been selected, the algorithm search for the optimal DAG whose moralization is a subgraph of the k-tree. 

Recent structural learning algorithms with \textit{unbounded} treewidth \citep{scanagatta2015a} can cope with thousands of variables. 
Yet the unbounded treewidth provides no guarantee about the complexity of the inferences of the inferred models.
We aim at filling this gap, learning treewidth-bounded Bayesian network models
in domains with thousands of variables. 

Structural learning is usually accomplished in two steps: \emph{parent set identification} and 
\emph{structure optimization}.
Parent set identification produces a list of suitable candidate parent sets for  each variable.
\emph{Structure optimization} assigns a parent set to each node, maximizing the score 
of the resulting structure without introducing cycles. 

Our first contribution regards parent set identification. 
We provide a bound for pruning the sub-optimal parent sets when dealing with the BIC score;
the bound is often tighter than the currently published ones \citep{decampos2011a}.

As a second contribution, we propose two approaches for learning Bayesian networks with bounded treewidth.
They are based on an iterative procedure which is able to add new variables to the current structure, maximizing the resulting score and respecting the treewidth bound.

We compare experimentally our novel algorithms against S2 and S2+, which represent the state of the art on datasets with dozens of variables.
Moreover, we present results for domains involving up to \textit{ten thousand} variables, providing an increase
of two order of magnitudes with respect to the results published to date.
Our novel algorithms consistently outperform the competitors.

% Where we repeat what everyone already know.
\section{Structural learning}

Consider the problem of learning the structure of a Bayesian Network from a complete data set of $N$ instances $\mathcal{D} = \{D_1, ..., D_N\}$. 
The set of $n$ categorical random variables  is $\mathcal{X}=\{X_1, ..., X_n\}$. The goal is to find the best DAG $\mathcal{G} = (V, E)$, where $V$ is the collection of nodes and $E$ is the collection of arcs.
$E$ can be represented by the set of parents ${\Pi_1, ..., \Pi_n}$ of each variable. 

Different scores can be used to assess the fit of a DAG.
We adopt the {\it Bayesian Information Criterion} (or simply $\mathrm{BIC}$),
which asymptotically approximates the posterior probability of the DAG under
common assumptions.
The $\mathrm{BIC}$ score is \emph{decomposable}, being constituted
by the sum of the scores of the individual variables:  %$\BIC(\mathcal{G})=$
\begin{align*}
& \BIC(\mathcal{G})=\\
& =\sum_{i=1}^{n} \BIC(X_i,\Pi_i)=\sum_{i=1}^n
  \left(\LL(X_i|\Pi_i) + \PL(X_i,\Pi_i)\right)\, , \\
& \LL(X_i|\Pi_i) =\displaystyle\sum\nolimits_{\pi \in |\Pi_i|, ~ x \in
  |X_i|} N_{x,\pi}\log\hat{\theta}_{x|\pi}\, ,\\
& \PL(X_i,\Pi_i) = -\frac{\log N}{2}(|X_i|-1)(|\Pi_i|)\, ,
\end{align*}
% \begin{align*}
% & \mathrm{BIC}(\mathcal{G})=\displaystyle \sum\nolimits_{i=1}^{n} \mathrm{BIC} (X_i, \Pi_i) =\\
% & \displaystyle \sum\nolimits_{i=1}^{n} \displaystyle \sum\nolimits_{\pi \in |\Pi_i|}  \displaystyle \sum\nolimits_{x \in |X_i|} N_{x,\pi}\log\hat{\theta}_{x|\pi}- \frac{\log N}{2}(|X_i|-1)(|\Pi_i|)\, ,
% \end{align*}
where $\hat{\theta}_{x|\pi} $ is the maximum likelihood estimate of the conditional probability $P(X_i=x|\Pi_i=\pi)$, and $N_{x,\pi}$ represents the number of times $(X=x\land\Pi_i=\pi)$ appears in the data set, and $|\cdot|$ indicates the size of the Cartesian product space of the variables given as argument.
Thus $|X_i|$ is the number of states of $X_i$ and $|\Pi_i|$ is the product
of the number of states of the parents of $X_i$.

Exploiting decomposability, we first identify independently for each variable a
list of candidate parent sets (the parent set identification task). Later, we
select for each node the parent set that yields the highest-scoring
treewidth-bounded DAG, which we call structure optimization.

\subsection{Parent sets identification}
When learning with limited treewidth it should be noted 
that the number of parents is a lower bound for the treewidth, since a node and its parents form a clique in the moralized graph. 
Thus, before running the structure optimization task, the list of candidate
parent sets of each node has to include parent sets with size up to $k$, if the
treewidth has to be bounded by $k$ (the precise definition of treewidth will be
given later on).
In spite of that, for values of $k$ greater than 3 or 4, we cannot compute all candidate
parent sets, since it already has time complexity $\Theta(N\cdot n^{k+1})$. In this section we
present the first contribution of this work: a bound for BIC scores that can be
used to prune their evaluations while processing all parent set candidates.
We first need a couple of auxiliary results.

\begin{Lemma}
Let $X$ be a node of $\mathcal{X}$, and $\Pi=\Pi_1\cup \Pi_2$ be a
parent set of $X$ such that $\Pi_1\cap\Pi_2=\emptyset$ and $\Pi_1,\Pi_2\neq\emptyset$.
Then \label{lem0} $\LL(X|\Pi) =$
\[
=\LL(X|\Pi_1) + \LL(X|\Pi_2) - \LL(X) + N\cdot \mathrm{ii}(X;\Pi_1;\Pi_2),
\]
\noindent
where $\mathrm{ii}$ is the {\it Interaction Information} estimated from data.
\end{Lemma}
\begin{proof}
It follows trivially from Theorem 1 in~\citep{scanagatta2015a}.
\end{proof}

It is known that $\LL(\Pi_1) \leq N\cdot\mathrm{ii}(\Pi_{1}; \Pi_{2}; X)\leq
-\LL(\Pi_1)$, and that the order of arguments is irrelevant (that is,
$\mathrm{ii}(\Pi_{1}; \Pi_{2}; X)=\mathrm{ii}(\Pi_{2}; \Pi_{1}; X)=\mathrm{ii}(X; \Pi_{1}; \Pi_{2})$).
These inequalities provide bounds for the log-likelihood in line with the result presented in Corollary
1 of ~\citep{scanagatta2015a}. We can manipulate that result to obtain new
tigher bounds.

\begin{Lemma}
Let $X,Y_1,\ldots,Y_t$ be nodes of $\mathcal{X}$, and $\Pi\neq\emptyset$ be a
parent set for $X$ with $\Pi\cap\mathcal{Y}=\emptyset$, where $\mathcal{Y}=\{Y_1,\ldots,Y_t\}$. Then \label{lem1}
$\LL(X|\Pi\cup\mathcal{Y}) \leq \LL(X|\Pi) + \sum_{i=1}^t w(X,Y_i)$, where
$w(X,Y_i)=\MI(X,Y_i) - \max\{\LL(X);\LL(Y_i)\}$,
% \begin{align*}
% &\LL(X|\Pi\cup\mathcal{Y}) \leq \LL(X|\Pi) + \sum_{i=1}^t w(X,Y_i)\, ,\\
% &w(X,Y_i)=\MI(X,Y_i) - \max\{\LL(X);\LL(Y_i)\},
% \end{align*}
\noindent where $\MI(X,Y_i)=\LL(X|Y_i) - \LL(X)$ is the empirical mutual information.
\end{Lemma}
\begin{proof}
It follows from the bounds of $\mathrm{ii}(\cdot)$ and the successive application of Lemma~\ref{lem0} to $\LL(X|\Pi\cup\mathcal{Y})$, taking out one node of $\mathcal{Y}$ a time.
\end{proof}

The advantage of Lemma~\ref{lem1} is that $\MI(X,Y_i)$ and $\LL(X)$ and $\LL(Y_i)$ (and hence $w(X,Y_i)$) can be all precomputed efficiently in total time $O(N\cdot n)$ for a given $X$, and
since BIC is composed of log-likelihood plus penalization (the latter is
efficient to compute), we obtain a new means of bounding BIC scores as follows.

\begin{Theorem}
Let $X\in\mathcal{X}$, and $\Pi\neq\emptyset$ be
a parent set for $X$, $\Pi_0=\Pi\cup\{Y_0\}$ for some
$Y_0\in\mathcal{X}\setminus\Pi$, and
$Y'=\max_{Y\in\mathcal{X}\setminus\Pi_0} \left(w(X,Y)+\PL(X,\Pi\cup\{Y\})\right)$.
If $w(X,Y_0) +\PL(X,\Pi_0)\leq \PL(X,\Pi)$ and
$w(X,Y') +\PL(X,\Pi\cup\{Y'\})\leq 0$, with $w(\cdot)$ as defined in Lemma~\ref{lem1},
then $\Pi_0$ and any of its supersets are not optimal.
 \label{lem2}\end{Theorem}
\begin{proof}
Suppose $\Pi'=\Pi_0\cup\mathcal{Y}$, with $\mathcal{Y}=\{Y_1,\ldots,Y_t\}$ and
$\mathcal{Y}\cap\Pi_0=\emptyset$ ($\mathcal{Y}$ may be empty). 
We have that
\begin{align*}
\BIC&(X,\Pi')= LL(X|\Pi') + \PL(X,\Pi')\\
\leq & LL(X|\Pi') + \PL(X,\Pi_0)+\sum_{i=1}^t \PL(X,\Pi\cup\{Y_i\})\\
\leq &\LL(X|\Pi) +\PL(X,\Pi_0)+w(X,Y_0) +\\
& \sum_{i=1}^t \left(w(X,Y_i) + \PL(X,\Pi\cup\{Y_i\})\right)\\
%\leq &\LL(X|\Pi) +\PL(X,\Pi) + \sum_{i=1}^t \left(w(X,Y_i) + \PL(X,\Pi\cup\{Y_i\})\right)\\
%\leq &\BIC(X,\Pi) + \sum_{i=1}^t \left(w(X,Y_i) + \PL(X,\Pi\cup\{Y_i\})\right)\\
\leq &\BIC(X,\Pi) + t\left(w(X,Y') + \PL(X,\Pi\cup\{Y'\})\right)\\
\leq &\BIC(X,\Pi).
\end{align*}
First step is the definition of BIC, second step uses the fact that the penalty
function is exponentially fast with the increase in number of parents, third
step uses Lemma~\ref{lem1}, fourth step uses the assumptions of the theorem and
the fact that $Y'$ is maximal. Therefore we would choose $\Pi$ in place of
$\Pi_0$ or any of its supersets.
\end{proof}

Theorem~\ref{lem2} can be used to discard parent sets during already their
evaluation and without the need to wait for precomputing all possible
candidates. We point out that these bounds are new and not trivially
achievable by current existing bounds for BIC. As a byproduct, we obtain bounds
for the number of parents of any given node.

\begin{Corollary}
Using BIC score, each node has at most 
$O(\log N - \log\log N)$ parents in the optimal structure.
\label{lem3}
\end{Corollary}
\begin{proof}
Let $X$ be a node of $\mathcal{X}$ and $\Pi$ a possible parent set. Let $Y\in\mathcal{X}\setminus\Pi$.
From the fact that
$\MI(X,Y) \leq \log |X|$, and $\max\{LL(X);LL(Y)\}
\geq -N\cdot \log |X|$, we have that $w(X,Y)\leq (N+1) \log |X|$,
 with $w(\cdot)$ as defined in Lemma~\ref{lem1}. Now
\begin{align*}
&\log |\Pi| \geq \log\left(\frac{2 \log |X|}{|X|-1}\right) + \log\left(\frac{N+1}{\log N}\right)\iff\\
&(N+1) \log |X| \leq \frac{\log N}{2}\cdot |\Pi|(|X|-1) \Longrightarrow\\
&w(X,Y)\leq -\PL(X,\Pi\cup\{Y\})+ \PL(X,\Pi)
\end{align*}
\noindent for any $Y$, and so by Theorem~\ref{lem2}
no super set of $\Pi$ is optimal. Note that $\log |\Pi|$ is greater than or
equal to the number of parents in $\Pi$, so we have
proven that any node in the optimal structure has at most $O(\log N - \log\log
N)$, which is similar to previous known results (see e.g.~\citep{decampos2011a}).
%-- even though the authors there neglect the term $\log\log N$.
\end{proof}

\subsection{Treewidth and $k$-trees}

We use this section to provide the necessary definitions and notation.\\
\paragraph{Treewidth}
We illustrate the concept of treewidth following the notation
of \citep{elidan2008}.
We denote an undirected graph as $\mathcal{H}=(V,E)$ 
where $V$ is the vertex set and $E$ is the edge set.
A \textit{tree decomposition} of $H$ 
is a pair ($\mathcal{C},\mathcal{T}$) where $\mathcal{C}=\{C_1,C_2,...,C_m\}$ 
is a collection of subsets of $V$ and $T$ is a tree on $\mathcal{C}$, so that:
\begin{itemize}
	\item  $ \cup_{i=1}^{m} \,\, C_i=V$;
	\item for every edge which
	connects the vertices $v_1$ and $v_2$, 
	 there is a subset $C_i$ which contains both $v_1$ and $v_2$;	
	\item for all $i,j,k$ in $\{1,2,..m \}$ if $C_j$ is on the  path between 
	$C_i$ and $C_k$ in $\mathcal{T}$ then $C_i \cap C_k  \subseteq C_j$.
\end{itemize}
The width of  a tree decomposition is $\max(|C_i|)-1$ where
$|C_i|$ is the number of vertices in $C_i$.
The treewidth of $H$ is the minimum width 
among all possible tree decompositions of $G$.	

The treewidth can be equivalently defined
in terms of triangulation of $\mathcal{H}$.
A triangulated graph is an undirected graph in which every cycle of length greater than three contains a chord.
The treewidth of a triangulated graph is the size of the maximal clique of the graph minus one.
The treewidth of  $\mathcal{H}$ is the minimum
treewidth over all the possible triangulations of $\mathcal{H}$.

The treewidth of a Bayesian network is characterized with respect to all possible triangulations of its moral graph.
The moral graph $M$ of a DAG  is an undirected graph that includes an edge ($i\rightarrow j$) for every edge ($i\rightarrow j$) in  the DAG and an edge ($p \rightarrow q$) for every pair of edges ($p\rightarrow i$), ($q\rightarrow i$) in the DAG.
The treewidth of a DAG is
the minimum
treewidth over all the possible triangulations of its moral graph $\mathcal{M}$.
Thus the maximal clique of any moralized triangulation of $\mathcal{G}$
is an upper bound on the treewidth of the model.

\paragraph{$k$-trees}
An undirected graph $T_k=(V, E)$ is a $k$-tree
if it is a \textit{maximal} graph of tree-width $k$:
any edge added  to $T_k=(V, E)$ increases its treewidth.  

A $k$-tree is inductively defined as follows \citep{patil1986}.
Consider a ($k+1$)-clique, namely
a complete graph with $k+1$ nodes.
A ($k+1$)-clique is a $k$-tree.

A ($k+1$)-clique can be decomposed into multiple $k$-cliques.
Let us denote by $z$ a node not yet included in the list of vertices $V$.
Then the graph obtained by connecting $z$ to every node 
of a $k$-clique of $T_k$ is also a $k$-tree. 

The treewidth of 
any subgraph of a $k$-tree (\textit{partial $k$-tree}) is bounded by $k$.
Thus a DAG whose triangulated moral graph is subgraph of a $k$-tree 
has treewidth bounded by $k$.

% Where we give o\usepackage[round]{natbib}ur wonky explanation, hoping to hide sufficiently well under formulas and moar formulas.
\section{Incremental treewidth-bounded structure learning}

We now turn our attention to the structure optimization task.
Our approach proceeds by repeatedly sampling an order $ \prec $ over the variables and then 
identifying the highest-scoring DAG with bounded-treewidth consistent with the order.
The size search space of the possible orders is $n!$, thus smaller than the search space of the possible k-trees. 
Once the order is sampled, we incrementally learn the DAG; it is guaranteed that
at each step the moralization of the DAG is a subgraph of a $k$-tree. The
treewidth of the DAG eventually obtained is thus bounded by $k$. The algorithm proceeds as follows.

\paragraph{Initialization}
The initial k-tree $\mathcal{K}_{k+1}$ is constituted by 
the complete clique over the first $k+1$ variables in the order.
The initial DAG  $\mathcal{G}_{k+1}$ is learned over the same  $k+1$ variables.
Since ($k+1$) is a small number of variables, we can exactly learn $\mathcal{G}_{k+1}$. In particular we adopt the method
of \cite{cussens11}.
The moral graph of  $\mathcal{G}_{k+1}$ is a subgraph of $\mathcal{K}_{k+1}$ and thus $\mathcal{G}_{k+1}$ has bounded treewidth.

\paragraph{Node's addition}
We then iteratively add each remaining variable. 
Consider the next variable in the order, $X_{\prec i}$, where $i \in \{ k+2, ..., n \}$.
Let us denote by  $\mathcal{G}_{i-1}$ and $\mathcal{K}_{i-1}$ 
the DAG and the k-tree which have to be updated by adding $X_{\prec i}$. 
We add $X_{\prec i}$ to $\mathcal{G}_{i-1}$, under the constraint that its parent set
$\Pi_{\prec i}$ is a subset of a complete $k$-clique in $\mathcal{K}_{i-1}$. 
This yields the updated DAG $\mathcal{G}_{i}$. 
We then update the k-tree, connecting 
$X_{\prec i}$ to such $k$-clique.
This yields the updated k-tree
$\mathcal{K}_i$; it contains an additional
$k+1$-clique compared to $\mathcal{K}_{i-1}$.  
By construction, $\mathcal{K}_i$ is also a $k$-tree. 
The moral graph of  $\mathcal{G}_i$ cannot add arc outside this $(k+1)$-clique; thus it is a subgraph of $\mathcal{K}_i$.

\paragraph{Pruning orders}
Notice that $\mathcal{K}_{k+1}$ and $\mathcal{G}_{k+1}$
depend only on which are the first $k+1$ variables and not on their relative positions.
Thus all the orders which differ only as for the relative 
position of the first $k+1$ elements are \textit{equivalent}
for our algorithm.
Thus  once we have sampled an order  and identified the corresponding DAG, we 
can prune the remaining $(k+1)!-1$ equivalent orders.

In order to choose the parent set to be assigned to each variable added to the graph we propose two algorithms: k-A* and k-G.

\subsection{k-A*}

We formulate the problem as a \emph{shortest path finding} problem.  
We define each state as a step towards the completion of the structure, where a new variable is added to the DAG $\mathcal{G}$. 
Given $X_{\prec i}$ the variable assigned in the state $S$, we define a successor state of $S$ for each $k$-clique we can choose for adding the variable $X_{\prec i+1}$.
The approach to solve the problem is based on a path-finding A* search, with cost function for state $S$ defined as $ f(S) = g(S) + h(S)$.  The goal is the state minimizing $f(S)$ where all the variable have been assigned. 

$g(S)$ is the cost from the initial state to $S$, and we define it as the sum of scores of already assigned parent sets: 
\begin{align*}
 g(S) = \sum\limits_{j=0}^{i} score(X_{\prec j}, \Pi_{\prec j})\, .
\end{align*}

$h(S)$ is the estimated cost from $S$ to the goal. It is the sum of best assignable parent sets for the remaining variables. 
Note that we know that $X_a$ can have $X_b$ as parent only if $X_b \prec X_a$: 
\begin{align*}
 g(S) = \sum\limits_{j=i+1}^{n} best(X_{\prec j})\, .
\end{align*}

The algorithm uses an \emph{open} list to store the search frontier. 
At each step it recovers the state with the smallest $f$ cost, generate the successors state and insert them into \emph{open}, until the optimal is found. 

The A* approach requires the $h$ function to be \emph{admissible}. 
The function h is  admissible if the estimated cost is never greater than the true cost to the goal state.
Our approach guarantees this property since the true cost of each step (score of chosen parent set for $X_{\prec i+1}$) is always equal or greater than the estimated (score of best selectable parent set for $X_{\prec i+1}$). 

We also have that $h$ is \emph{consistent}, meaning that for any state $S$ and its successor $T$, $h(S) \leq h(T)+c(S, T)$, where $c(S, T)$ is the cost of the edges added in $T$. 
This follows from the previous argument. 
Now we have that $f$ is monotonically non-decreasing on any path, and the algorithm is guaranteed to find the optimal path as long as the goal state is reachable.

% Given an ordering, the algorithm has worst-case complexity $\prod_{i=0}^{n-(k+1)} i*k+1$. 
% At each new step there are $k$ new $k$-cliques added to the underlying $k$-tree. 
% For the state $i$ this results on $i*k+1$ successor state. 
% In the worst case, K-A* would have to explore all of them. 
% 
% Compare this with the worst-case complexity of S2+ of $\prod_{i=0}^{n-(k+1)} (i*k+1)*(n-(k+i))$. 
% The number of successor state is increased by the choice on the next variable to process. 
% For the state $i$ this results on $(i*k+1)*(n-(k+i))$.

% We also 
% 
% We have that $h(n)$ is both consistent (non-descending) and admissible (always optimistic) (key requirement for A* to work). For the consistent I'll need proof (the estimation cost is intrinsically determined by each step, is it sufficient?)
% Now: 

% \begin{itemize}
%  \item It starts with the initial state, forming an ($k+1$)-clique with by the $k+1$ initial variables 
%  \item It forms the $k$ successors states that can be produced by the initial state by choosing a legit $k$ clique to which assign the next variable, $X_{k+2}$, and put them into \emph{open} list. 
%  \item Until the optimal is found, do the following: get the best state in \emph{open} (best means with lower ), produce the successors and put them into \emph{open}. There is a successor state for each handler (=$k$-clique) to which you can attach the next variable in the order. 
% \end{itemize}

\subsection{k-G}

In some cases a high number of variables or a high treewidth 
prevent the use of k-A*.
We thus propose a greedy alternative approach, K-G. 
Following the path-finding problem defined previously, it takes a greedy approach: at each step chooses for the variable $X_i$ the highest-scoring parent set that is subset of an existing $k$-clique in $\mathcal{K}$.

\subsection{Space of learnable DAGs}
A \emph{reverse topological order} is an order $\{v_1 , . . . v_n \}$ over the vertexes $V$ of a DAG in which each $v_i$ appears before its parents $\Pi_{i}$.
The search space of our algorithms is restricted to the DAGs whose reverse topological order, when used as variable elimination order, has treewidth $k$. 
This prevents recovering DAGs which have bounded treewidth but lack this property.

We start by proving by induction that the reverse topological order has treewidth $k$ in the DAGs recovered by our algorithms.
Consider the incremental construction of the DAG previously discussed.

The initial DAG $\mathcal{G}_{k+1}$ is induced over $k+1$ variables; thus every elimination ordering has treewidth bounded by $k$.

For the inductive case, assume that $\mathcal{G}_{i-1}$ satisfy the property.
Consider the next variable in the order, $X_{\prec_{i}}$, where $i \in \{ k+2, ..., n \}$.
Its parent set $\Pi_{\prec_{i}}$ is a subset of a $k$-clique in $\mathcal{K}_{i-1}$.
The only neighbors of $X_{\prec_{i}}$ in the updated DAG $\mathcal{G}_i$
are its parents $\Pi_{\prec_{i}}$.
Consider performing variable elimination on the
the moral graph of $\mathcal{G}_i$, using a 
reverse topological order.
Then $X_{\prec_{i}}$ will be eliminated before $\Pi_{\prec_{i}}$, without introducing fill-in edges. Thus the treewidth associated to any reverse topological order is bounded by $k$.
This property inductively applies to the addition also of the following nodes 
up to $X_{\prec_{n}}$.

\paragraph{Inverted trees}
An example of DAG non recoverable by our algorithms is the 
specific class of polytrees that we call \textit{inverted trees}, that is, DAGs
with indegree equal to one.
An inverted tree with $m$ levels and treewidth $k$ can be built as follows.
Take the root node (level one) and connect it to $k$ child nodes (level two).
Connect each node of level two to $k$ child nodes (level three).
Proceed in this way up to the m-th level and then invert the direction of all the arcs.

Figure \ref{fig:reverse} shows an inverted tree with $k$=2 and $m$=3.  
It has treewidth two, since its moral graph is constituted by the cliques
\{A,B,E\}, \{C,D,F\}, \{E,F,G\}.
The treewidth associated to the reverse topological order is instead three, using
the order G, F, D, C, E, A, B.

\begin{figure}[!ht]
	\centering
	\begin{tikzpicture}		  
	
	\node[draw, circle] (A) at (0, 2) {A};
	\node[draw, circle] (B) at (1, 2) {B};
	\node[draw, circle] (C) at (2, 2) {C};
	\node[draw, circle] (D) at (3, 2) {D};
	\node[draw, circle] (E) at (0.5, 1) {E};	 
	\node[draw, circle] (F) at (2.5, 1) {F};	
	\node[draw, circle] (G) at (1.5, 0) {G};
	
	\draw[->, line width=1pt] (A) to (E);
	\draw[->, line width=1pt] (B) to (E);
	\draw[->, line width=1pt] (C) to (F);
	\draw[->, line width=1pt] (D) to (F);
	\draw[->, line width=1pt] (E) to (G);
	\draw[->, line width=1pt] (F) to (G);
	\end{tikzpicture}
	
	\caption{Example of inverted tree.}
	\label{fig:reverse}
\end{figure}
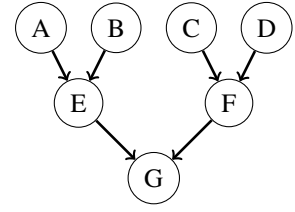	      

If we run our algorithms with bounded treewidth $k$=2, it will be unable to recover the actual inverted tree. It will instead identify a high-scoring DAG whose
reverse topological order has treewidth 2. 

\subsection{Our implementation of S2 and S2+}

Here we provide the details of our implementation of S2 and S2+. 
They both use the notion of Informative Score \citep{NieCJ15}, an approximate measure of the fitness of a k-tree. The I-score of a k-tree $T_k$ is defined as 
\begin{align*}
 IS(T_k) = \frac{S_{mi}(T_k)}{|S_l(T_k)|}\, ,
\end{align*}
where $S_{mi}(T_k)$ measures the expected loss of representing the data with the k-tree. Let $I_{ij}$ denote the mutual information of node $i$ and $j$:
\begin{align*}
S_{mi}(T_k) = \sum_{i, j} I_{ij} - \sum_{i, j \notin T_k} I_{ij}\, .
\end{align*}
$S_l(T_k)$ instead is defined as the score of the best pseudo subgraph of the k-tree by dropping the acyclic constraint: 
\begin{align*}
S_l(T_k) = \max_{m(G) \in T_k} \sum_{i \in N} score(X_i, \Pi_i)\, ,
\end{align*}
where $m(G)$ is the moral graph of DAG $G$, and $score(X_i, \Pi_i)$ is the local score function of variable $X_i$ for the parent set $\Pi_i$.

The first phase of both S2 and S2+ consists in a k-tree sampling.
In particular, S2 obtains k-trees by using the Dandelion sampling discussed in \citep{NieMCJ14}. The proposed k-trees are then accepted with probability: 
\begin{align*}
 \alpha = min \left( 1, \frac{IS(T_k)}{IS(T^*_k)} \right)\, ,
\end{align*}
where $T^*_k$ is the current k-tree with the largest I-score \citep{NieCJ15}. 

Instead S2+ selects the $k+1$ variables with the largest I-score and finds the k-tree maximizing the I-score from this clique, as discussed in \citep{nieCJ16}. Additional k-trees are obtained choosing a random initial clique. 

The second phase of the algorithms looks for a DAG whose moralization is subgraph of the chosen k-tree.
For this task, the authors proposed an approximate approach based on partial order sampling (Algorithm 2 of \citep{NieMCJ14}). 
In our experiments, we found that using Gobnilp for this task yields slightly higher scores, therefore we adopt this approach in our implementation.
We believe that it is due to the fact that constraining the structure
optimization to a subjacent graph of a k-tree results in a small number of allowed arcs for the DAG.
% This results in a first LP solution with relatively a small number of cycles. Cutting planes can then be easily added. 
Thus our implementation finds the highest-scoring DAG 
whose moral graph is a subgraph of the provided k-tree. 

\subsubsection{Discussion}

% A similar approach to our algorithm is discussed in \citep{NieCJ15}. 
% These approaches are similar in the fact that they both employ the notion of a k-tree.
% The fundamental difference is that they perform k-tree sampling and then they search for the best partial k-tree, while we perform a direct structure optimization under additional constraint that ensure it remains a partial k-tree.

The problem with k-tree sampling is that each k-tree enforces a random constraint over the arcs that may appear in the final structure.
The chance that we randomly sample a k-tree that allows good scoring arcs
becomes significantly smaller as the number of variables increases, and the
space of possible k-tree increases as well.
The criterion for probabilistic acceptance, presented in the past section, has
been proposed for tackling this issue, but it does not resolve the situation completely.

Our approach instead focus immediately on selecting the best arcs, in a way that guarantees the treewidth bound. 
Experimentally we observed that k-tree sampling is quicker, producing an higher number of candidate DAGs, whose scores are unfortunately low. Our approach instead generates less but higher-scoring DAGs. 

\citep{nieCJ16} improves on the notion of k-tree, searching for the optimal one
with respect to the \emph{Informative Score} (IS). IS considers only the mutual
information between pair of variables, and it may exaggerate the importance of
assigning some arcs. The IS criterion may suggest parents for a node with
separately have high mutual information but are bad together as a parent set. %Thus a k-tree with a high IS may induce a low-scoring DAG. 

% \begin{figure}[h]
%   \centering
%      \begin{tikzpicture} 
% 
% 	\begin{axis}[
% 	  %enlargelimits=false,   
% 	  width=0.35\textwidth,
% 	  height=12em,
% 	  %xtick={100, 500, 900},
% 	  at={(0,0)},
% 	  xlabel={Number of variables},
% 	  ylabel={Time in seconds},
% 	  % x label style={at={(axis description cs:0.5,-0.2)},anchor=north}
% 	]
% 	\addplot+[mark size=1pt] table{data/s2+.dat};
% 	\end{axis} 
%     \end{tikzpicture}
% 	\caption{Mean time required for finding one optimal k-tree with S2+ ($k = 4$)}
% 	\label{fig:time_2+_vars}    
% \end{figure}
% 
% \begin{figure}[h]
%   \centering
%      \begin{tikzpicture} 
% 
% 	\begin{axis}[
% 	  %enlargelimits=false,   
% 	  width=0.35\textwidth,
% 	  height=12em,
% 	  %xtick={100, 500, 900},
% 	  at={(0,0)},
% 	  xlabel={treewidth},
% 	  ylabel={Time in seconds},
% 	  % x label style={at={(axis description cs:0.5,-0.2)},anchor=north}
% 	]
% 	\addplot+[mark size=1pt] table{data/s2+tw.dat};
% 	\end{axis} 
%     \end{tikzpicture}
% 	\caption{Mean time required for finding one optimal k-tree with S2+ ($n = 100$)}
% 	\label{fig:time_2+_tw}    
% \end{figure}

% Where we slightly misinterpret our result in hope of being accepted.
\section{Experiments}
We compare k-A*, k-G, S2 and S2+ in various experiments. 
We compare them through an indicator which we call \emph{W-score}:
the percentage of worsening of the BIC score of the selected treewidth-bounded method compared to the score of the Gobnilp  
solver \citep{cussens11}. 
Gobnilp achieves higher score than the treewidth-bounded methods since it has no limits on the treewidth.
Let us denote by $G$ the BIC score achieved by Gobnilp and by 
$T$ the BIC score obtained by the given treewidth-bounded method.
Notice that both $G$ and $T$ are negative.
The W-score is $W=\frac{G-T}{G}$. 
W stands for worsening and thus lower values of $W$ are better.
The lowest value of W is zero, while there is no upper bound on the value of W.

\begin{table*}[!ht]
	\rowcolors{2}{white}{lightblue}
	\centering
	\begin{tabular}{c c r r r r r}\toprule
		
		DATASET & VAR. 	& GOBNILP 			& S2 				& S2+ 				& k-G			& k-A* \\ \midrule
		nursery & 	9 & 	\f{-72159} & 	\f{-72159} & 	\f{-72159} & 	\f{-72159} & 	\f{-72159} \\  
		breast & 	10 & 	\f{-2698} & 	\f{-2698} & 	\f{-2698} & 	\f{-2698} & 	\f{-2698} \\ 
		housing & 	14 & 	-3185 & 	-3252 & 	-3247 & 	-3206 & 	\f{-3203} \\  
		adult & 	15 & 	-200142 & 	-201235 & 	-200926 & 	-200431 & 	\f{-200363} \\	  
		letter & 	17 & 	-181748 & 	-189539 & 	-186815 & 	-183369 & 	\f{-183241} \\  
		zoo & 	17 & 	-608 & 	-620 & 	-619 & 	-615 & 	\f{-613} \\  
		mushroom & 	22 & 	-53104 & 	-68670 & 	-64769 & 	-57021 & 	\f{-55785} \\  
		wdbc & 	31 & 	-6919 & 	-7213 & 	-7209 & 	-7109 & 	\f{-7088} \\  
		audio & 	62 & 	-2173 & 	-2283 & 	-2208 & 	-2201 & 	\f{-2185} \\  
		community & 	100 & 	-77555 & 	-107252 & 	-88350 & 	-82633 & 	\f{-82003} \\  
		hill & 	100 & 	-1277 & 	-1641 & 	-1427 & 	-1284 & 	\f{-1279} \\  
		
		\bottomrule
	\end{tabular}
	\caption{Comparison between bounded-treewidth structural learning algorithms on the data sets already analyzed by \citep{nieCJ16}.The highest-scoring solution with limited treewidth is boldfaced. In the first column as term of comparison we report the score of the solution obtained by Gobnilp without bound on the treewidth.	
		\label{table:networks-exp-datasets}}
\end{table*}

\subsection{Learning inverted trees}

As already discussed our approach
cannot learn an inverted tree with $k$ parents per node
if given bounded treewidth $k$.
In this section we study their performance in this worst-case scenario.

We start with treewidth $k=2$. 
We consider the number of variables $n \in \{21,41,61,81,101\}$.
For each value of $n$ we
generate 5 different inverted trees. 
An inverted tree is generated by randomly selecting a root variable $X$ from the existing graph and adding $k$ new variables as $ \Pi_X $, until the graph contains $n$ variables. 
All variables are binary and we sample their conditional probability tables from a Beta(1,1).
We sample 10,000 instances from  each generated inverted tree.

We then perform structural learning with k-A*, k-G, S2 and S2+, setting $k=2$ as limit on the treewidth.
We allow each method to run for ten minutes.
Both S2 and S2+ could in principle recover the true structure, which is prevented to our algorithms.
The results are shown in Fig.\ref{fig:inverted}. 
Qualitatively similar results are obtained repeating the experiments with $k=4$.

\begin{figure}[!ht]
	\centering
	\includegraphics{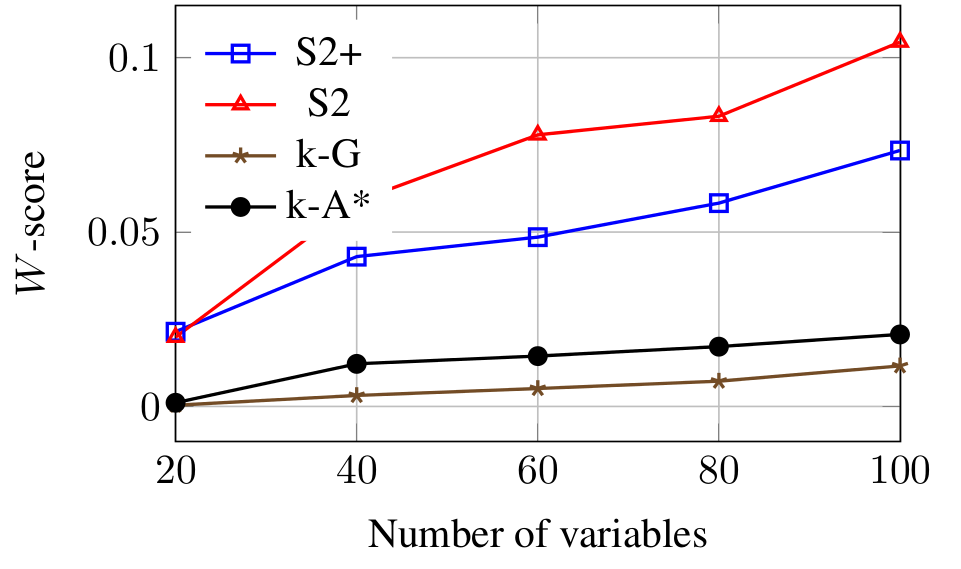}
	\caption{Structural learning results when the actual DAGs are inverted trees ($k$=2). Each point represent the mean W-score over 5 experiments. Lower values of the $W$-score are better.}
	\label{fig:inverted}    
\end{figure}

Despite the unfavorable setting, both k-G and k-A* yield DAGs with higher score than S2 and S2+, consistently for each value of $n$. Thus the limitation of the space of learnable DAGs 
does not hurt much the performance of k-G and k-A*. 
In fact S2 could \textit{theoretically} recover the actual DAG, but this would
require too many samples from the space of the k-trees, which is prohibitive.

\begin{table}[!ht]
	\rowcolors{2}{white}{lightblue}
	\centering
	\begin{tabular}{c c r r r r r}
		\toprule
		&		    S2 &		 S2+ &	k-G &	k-A* \\ 
		\midrule
		Iterations & 803150 &	3 &			7176 &		66		\\
		%		Mean & 		-116544 & 	-86459 & 	-84634 & 	-84322  \\  
		Median & 	-273600 & 	-267921 & 	-261648 & 	-263250  \\ 
		Max & 		-271484 & 	-266593 & 	-258601 &	-261474 \\
		\bottomrule
	\end{tabular}
	\caption{Statistics of the solutions yielded by different methods on an inverted tree ($n=100$, $k=4$).} 
	\label{tab:comp-inverted}
\end{table}

We further investigate the differences between methods 
by providing in Table \ref{tab:comp-inverted} some statistics about 
the candidate solutions they generate. 
\emph{Iterations} is the number of proposed solutions; for S2 and S2+ it is the number of explored k-trees, while for k-G and k-A* it is number of explored orders. 

During the execution, S2 samples almost one million k-trees.
Yet it yields the lowest-scoring DAGs among the different methods.
This can be explained considering that a randomly sampled k-tree has a low chance to cover 
a high-scoring DAG.
S2+ recovers only a few k-trees, but their scores are higher than those of S2. This 
confirms the effectiveness of driving the search for good k-trees through the Informative Score.
As we will see later, however, this idea does not scale on very large data sets.

As for our methods, k-G samples a larger number of orders than k-A* does
and this allows it to achieve higher scores, even if 
it sub-optimally deals with each single order.

% We believe this results from the fact that the scoring function measures the similarity between the joint distribution encoded in the original and the learned graph. 
% The original graph may have a learnable graph in its equivalence class, allowing our approach to effectively approximate its distribution.
% On the other hand, our approach may perform worse in terms of similarity between the graphs of the original and the learned structure.

\subsection{Small data sets}
We now present experiments on the data sets already considered by \citep{nieCJ16}.
They involve up to 100 variables. 
We set the bounded treewidth to $k=4$.
We provide each structural learning method with the same pre-computed scores of parent sets. 
We allow each method to run for ten minutes.
We perform 10 experiments on each data set and we report the median scores in Table~\ref{table:networks-exp-datasets}.
Our results are not comparable with those reported by \citep{nieCJ16} since we use the BIC
while they use BDeu. 

Remarkably both k-A* and k-G achieve higher scores than both S2 and S2+ do on almost all data sets.
Only on the smallest data sets all methods achieve the same score.
Between our two novel algorithms, k-A* has a slight advantage over k-G.

We provide statistics about the candidate solutions generated by each method in 
Table~\ref{tab:comp}. The results of the table refer
in particular to the \emph{community} data set ($n$=100).
The conclusions are similar to those of previous analyses.
S2 performs almost one million iterations, but they are characterized by low scores.
S2+ performs a drastically smaller number of iterations, but is able anyway 
to outperform S2. 
Similarly k-A* is more effective than k-G, despite generating a lower number of candidate solution.
The reduced number of candidate solutions generated by both S2+ and k-A* suggest that they cannot scale
on data sets much larger than those of this experiment.

\begin{table}[!ht]
	\rowcolors{2}{white}{lightblue}
	\centering
	\begin{tabular}{c c r r r r r}

		\toprule
		&		    S2 &		 S2+ &	k-G &	k-A* \\ 
		\midrule
		Iterations & 945716 &	3 &			3844 &		87		\\
		%		Mean & 		-116544 & 	-86459 & 	-84634 & 	-84322  \\  
		Median & 	-115887 & 	-85546 & 	-85332 & 	-84771  \\ 
		Max & 		-107840 & 	-85270 & 	-82863 &	-82452 \\
		\bottomrule
	\end{tabular}
	\caption{Statistics of the solutions yielded by different methods on the community data set ($n$=100).}
	\label{tab:comp}
\end{table}

\subsection{Large data sets}
We now consider  10 large data sets ($100\leq n \leq 400$) listed
in Table~\ref{tab:large-dsets}.

 \begin{table}[!ht]
 	\rowcolors{2}{white}{lightblue}
 	\centering
 	%\begin{tabular}{abab}
 	\begin{tabular}{crcr}
 		\toprule
 		Data set & $n$ & Data set & $n$  \\
 		\midrule
 		Audio     & 100 & Pumsb-star & 163    \\
 		Jester    & 100 & DNA        & 180  \\
 		Netflix   & 100 & Kosarek    & 190  \\
 		Accidents & 111 & Andes    & 223  \\
 		Retail    & 135 & MSWeb     & 294 \\
 		
 		\bottomrule
 	\end{tabular}
 	\caption{\textit{Large} data sets sorted according to the number of variables.
 		\label{tab:large-dsets}}
 	
 \end{table}
 
We consider the following treewidths: $k \in \{2,5,8\}$.
We split each data set randomly into three subsets.
Thus for each treewidth
we run 10$\cdot$3=30 structural learning experiments.

We provide all structural learning methods with the same pre-computed scores of parent sets and we let each method run for one hour.
For S2+, we adopt a more favorable approach, allowing it to run for one hour; if after one hour  the first k-tree was not yet solved, we allow it to run until it has solved the first k-tree.

In Table~\ref{tab:large-victories} we report how many times each method wins against another for each treewidth. The entries are boldfaced when the number of victories of an algorithm over another is  statistically significant according to the sign-test (p-value \textless 0.05).
Consistently for any chosen treewidth, k-G is significantly better than any competitor, including k-A*; moreover, k-A* is significantly better than both S2 and S2+.
 
 \begin{table}[!ht]
 	\rowcolors{2}{white}{lightblue}
 	\centering
 	\begin{tabular}{lllll}
 		\toprule
 		& k-A*              & S2                & S2+               &  \\
 		\midrule
 		k-G  & \textbf{29/20/24} & \textbf{30/30/29} & \textbf{30/30/30} &  \\
 		k-A* &                   & \textbf{29/27/20} & \textbf{29/27/21} &  \\
 		S2   &                   &                   & 12/13/\textbf{30} &  \\
 		\bottomrule
 	\end{tabular}
 	\caption{Result on the 30 experiments on large data sets. Each cell report how many times
 		the row algorithm yields a higher score than the column algorithm for treewidth 2/5/8.
 		For instance k-G wins on all the 40 data sets against S2 for each considered treewidth. }
 	\label{tab:large-victories}
 \end{table}
 
 This can be explained by considering that k-G explores more orders than k-A*, as for a given order it only finds 
an approximate solution. The results suggest that it is more important to
explore many orders instead of obtaining the optimal DAG given an order.

\subsection{Very large data sets}

%\begin{table}[!ht]
%		\rowcolors{2}{white}{lightblue}
%	\centering
%	\begin{tabular}{llllll}\toprule
%		        & k-A*              & S2                & S2+               & E-D               &  \\ \midrule
%		k-G     & \textbf{29/20/24} & \textbf{30/30/29} & \textbf{30/30/30} & \textbf{30/30/30} &  \\
%		k-A*    &                   & \textbf{29/27/20} & \textbf{29/27/21} & \textbf{29/27/21} &  \\
%		S2      &                   &                   & 12/13/\textbf{30} & 9/\textbf{25}/\textbf{30}  &  \\
%		S2+     &                   &                   &                   & 3/\textbf{20}/0            &  \\ \bottomrule
%%		E-D     &                   &                   &                   &  & 
%		%		k-A* & 1/9/6             &                   & \textbf{27/25/20} & \textbf{27/25/21} & \textbf{27/25/21} \\
%		%		S2   & 0/0/1             & 1/3/8             &                   & 10/12/\textbf{28} & 7/23/\textbf{28}  \\
%		%		S2+  & 0/0/0             & 0/0/0             & 18/16/0           &                   & 1/18/0            \\
%		%		E-D  & 0/0/0             & 0/0/0             & \textbf{21}/5/0   & 21/2/0            &
%	\end{tabular}
%	\caption{Result on the 30 large data sets. Each cell report how many times
%		the row algorithm yields a higher score than the column algorithm for treewidth 2/5/8.
%		For instance k-G wins on all the 40 data sets against S2 for each considered treewidth. \hl{update to sum 30}}
%	\label{tab:large-victories}
%\end{table}

As final experiment, we consider 14 very large data sets, containing more than 400 variables.
We include in these experiments three randomly-generated synthetic data sets 
containing 2000, 4000 and 10000 variables respectively.
These networks have been generated using the software BNGenerator \footnote{\url{http://sites.poli.usp.br/pmr/ltd/Software/BNGenerator/}}. 
Each variable has a number of states
randomly drawn from 2 to 4 and a number of parents randomly drawn from 0 to 6.
In this case, we perform 14$\cdot$3=42 structural learning experiments with each algorithm.
The only two algorithms able to cope with these data sets are k-G and S2.
Among them, k-G wins 42 times out of 42; this dominance is clearly significant.
This result is consistently found under each choice of treewidth ($k=$2, 5, 8). 
On average, the improvement of k-G over S2 fills about 60\% of the gap which separates S2 from the unbounded solver.

The W-scores of such 42 structural learning experiments are summarized
in Figure~\ref{fig:boxplots}. 
For both S2 and k-G, a larger treewidth allows to recover a higher-scoring graph.
In turn this decreases the W-score.
However k-G scales better than S2 with respect to the treewidth; its W-score
decreases more sharply with the treewidth.

\begin{table}[!ht]
	\rowcolors{2}{white}{lightblue}
\centering
	\begin{tabular}{crcr}
%\begin{tabular}{abab}
\toprule
 Data set & $n$ & Data set & $n$  \\
\midrule 
Diabets    & 	413 & 	C20NG      & 910 \\
Pigs   & 	441 & 	Munin & 1041  \\
Book      & 	500  & 	BBC        & 1058 \\
EachMovie & 	500 & 	 Ad         & 1556  \\
Link & 		724 & 	R2     & 2000   \\
WebKB & 	839  &  R4      & 4000   \\
Reuters-52 & 	889 & 	R10 & 10000 \\
\bottomrule
\end{tabular}
 \caption{ Very large data sets sorted according to the number $n$ of variables.}
\end{table}

\pgfplotsset{every tick label/.append style={font=\small}}
\begin{figure}[ht]
  \centering
\includegraphics{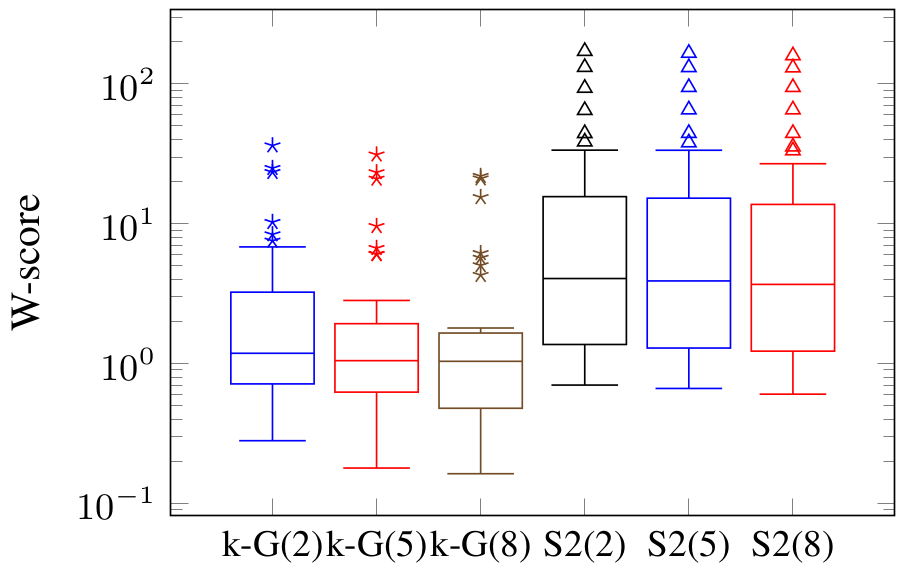}
	\caption{Boxplots of the W-scores, summarizing the results over 14$\cdot$3=42 structural learning experiments on very large data sets. Lower W-scores are better. The y-axis is shown in logarithmic scale.
	In the label of the x-axis we also report the adopted treewidth for each method: 2, 5 or 8.}
	\label{fig:boxplots}    
\end{figure}
  
It is interesting to analyze the statistics of the solutions generated
by the two methods. They are given in Table~\ref{tab:comp-mulin} for the data set Munin.
K-G generates a number of solutions which is a few orders of magnitude smaller than
that of S2. Yet, the scores of the obtained solutions are much higher.

\begin{table}[!ht]
	\rowcolors{2}{white}{lightblue}
	\centering
	\begin{tabular}{c c r r r r r}

		\toprule
		&		    S2 &	k-G  \\ 
		\midrule
		Iterations & 63637	& 	83\\
		%		Mean & 		-116544 & 	-86459 & 	-84634 & 	-84322  \\  
		Median & 	-6324236 & 	-3302131  \\ 
		Max & 		-6262538 & 	-2807518 \\
		\bottomrule
	\end{tabular}
	\caption{Statistics of the solutions yielded by different methods on the Munin data set ($n$=1041).}
	\label{tab:comp-mulin}
\end{table}
% \subsection{Further considerations}
% 
%   \begin{table*}[t]
%   	\begin{center}
%   		\begin{tabular}{ccc|ccc}
%   			\toprule
%   			Data set & $n$ & Time (sec.) & Data set & $n$ & Time (sec.)  \\
%   			\midrule
%   			Audio		& 100	& 4198.3	& MSWeb     	& 294 	&   	\\
%   			Jester		& 100	&	& Book      	& 500 	& $>24h$ 	\\
%   			Netflix   	& 100	& 4705.5	& EachMovie 	& 500 	& 	\\
%   			Accidents 	& 111	& 5148.3	& WebKB		& 839	& 	\\
%   			Retail     	& 135	&	& Reuters-52 	& 889	&	\\
%   			Pumsb-star 	& 163	&	& C20NG      	& 910	&	\\
%   			DNA		& 180	&	& BBC     	& 1058	&$>24h$	\\
%   			Kosarek		& 190	&	& Ad		& 1556	&$>24h$	\\
%   			\bottomrule
%   		\end{tabular}
%   		\caption{Time required to compute the first k-tree using A* Siqi Nie}
%   	\end{center}
%   	\vspace{-2em}
%   \end{table*}

% Where we pray for the mercy of the reviewers.
\section{Conclusion}

Our novel approaches for treewidth-bounded structural learning of
Bayesian Networks perform significantly better than state-of-the-art
methods. The greedy approach scales up to thousands of nodes and suggests
that it is more important to find good k-trees than to solve the internal structure
optimization task for each one of them. The methods consistently outperform
the competitors on a variety of experiments. All these methods and others
for unbounded learning of Bayesian networks can make use of our new bounds
for BIC scores in order to reduce the number of parent set evaluations during the precomputation of
scores. Further analyses of the bounds are left for future work.

%==ACK to be re-enabled if we get accepted!
%
% \subsection*{Acknowledgments}
% Work partially supported by the 
% Swiss NSF grants Nos.~200021\_146606~/~1 and~200020\_137680~/~1,
% and by the Swiss CTI project with Hoosh Technology.

\bibliographystyle{myplainnat}
\bibliography{biblio}

\end{document}